\documentclass[11pt]{article}

\usepackage{latexsym}
\usepackage[latin1]{inputenc}
\usepackage{amsmath, amssymb, graphicx}
\usepackage{paralist}
\usepackage{dsfont}
\usepackage{graphicx}
\usepackage{framed,xcolor,soul}
\usepackage[ruled,norelsize]{algorithm2e}
\usepackage[noend]{algorithmic}
\usepackage{calc}
\newcommand{\STATEWITHCOMMENT}[2]{\STATE{\makebox[\widthof{bla bla bla bla bla bla bla}][l]{#1}$\triangleright$ #2}}

\newcounter{mathcounter_d}
\newcounter{mathcounter_l}
\newcounter{mathcounter_c}
\newcounter{mathcounter_t}
\newenvironment{lemma}{\refstepcounter{mathcounter_l} \begin{trivlist} \item[\hskip \labelsep {\bfseries Lemma \arabic{mathcounter_l}.\enspace}] \it}{\end{trivlist}}

\newenvironment{theorem}{\refstepcounter{mathcounter_t} \begin{trivlist} \item[\hskip \labelsep {\bfseries Theorem \arabic{mathcounter_t}.\enspace}]}{\end{trivlist}}

\newcommand{\qed}{\nobreak \ifvmode \relax \else \ifdim\lastskip<1.5em \hskip-\lastskip \hskip1.5em plus0em minus0.5em \fi \nobreak \vrule height0.4em width0.5em depth0.25em\fi}
\newenvironment{proof}[1][Proof]{\begin{trivlist} \item[\hskip \labelsep {\bfseries #1}]}{\hfill\qed\end{trivlist}}

\newcommand{\N}{\mathbb{N}}
\newcommand{\R}{\mathbb{R}}
\newcommand{\Normal}{\mathcal{N}}
\newcommand{\Expectation}{\mathbb{E}}

\newcommand{\bmat}{\begin{pmatrix}}
\newcommand{\emat}{\end{pmatrix}}

\begin{document}

\title{The (1+1)-ES Reliably Overcomes Saddle Points}
\author{Tobias Glasmachers\\
		Institute for Neural Computation\\
		Department for Computer Science\\
		Ruhr-University Bochum, Germany\\
		\texttt{tobias.glasmachers@ini.rub.de}}
\date{}

\maketitle

\begin{abstract}
It is known that step size adaptive evolution strategies (ES) do not
converge (prematurely) to regular points of continuously differentiable
objective functions. Among critical points, convergence to minima is
desired, and convergence to maxima is easy to exclude. However,
surprisingly little is known on whether ES can get stuck at a saddle
point. In this work we establish that even the simple (1+1)-ES reliably
overcomes most saddle points under quite mild regularity conditions. Our
analysis is based on drift with tail bounds. It is non-standard in that
we do not even aim to estimate hitting times based on drift. Rather, in
our case it suffices to show that the relevant time is finite with full
probability.
\end{abstract}

\section{Introduction}

\begin{tabular}{lll}
\begin{minipage}{0.48\textwidth}
The question how optimization algorithms handle saddle points is a
classic subject. In the standard analysis of gradient-based
optimization, it is easy to rule out premature convergence to a regular
point. In contrast, excluding convergence to saddle points requires
considerable effort~\cite{dauphin2014identifying}.

In evolutionary computation, the situation is no different. Akimoto et
al.\ \cite{akimoto2010theoretical} established that many optimizers
cannot converge to a regular point of the objective function under the
rather basic assumption that they successfully diverge on a linear
slope.
\end{minipage}
&
\begin{minipage}{0.04\textwidth}
\end{minipage}
&
\begin{minipage}{0.48\textwidth}
\begin{center}
	\vspace*{-1cm}
	\includegraphics[width=0.9\textwidth]{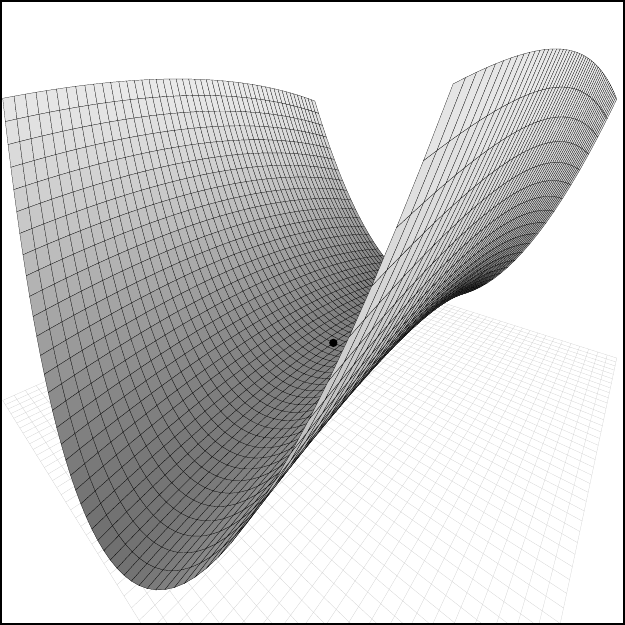}\\
	{\footnotesize \textbf{Fig.~1\setcounter{figure}{1}} Graph of a difficult saddle point.}
\end{center}
\end{minipage}
\end{tabular}\\

Prior work on the behavior of evolution strategies in the presence of a
saddle point seems to be sparse. We need to highlight that usually in
optimization the goal is \emph{not} to get stuck at a saddle point, but
rather to proceed to a (local) optimum. This is different from the goal
of locating saddle points by means of optimization techniques, in cases
where these saddles are of interest by
themselves~\cite{akimoto2021saddle}. That line of work on ``saddle point
optimization'', also called min-max-problems, is unrelated to our
research question.

In our own prior work \cite{glasmachers2020global}, we conducted a
detailed analysis of conditions under which convergence of the (1+1)-ES
to the global optimum can be guaranteed, on an extremely wide class of
functions. In that work, premature convergence to saddle points can only
be excluded if the success probability in the saddle point exceeds the
target success rate of $1/5$ in the limit of small step sizes. On the
other hand, for some extremely deceptive saddle points of sharp ridges,
a positive probability for premature convergence is proven.

There is a considerable gap between the two cases. While existing
guarantees do not apply to these cases, empirical evidence
indicates---maybe surprisingly---that already the simple (1+1)-ES
reliably overcomes even extremely ill-conditioned saddle points. In the
present paper we close this gap by cementing the empirical evidence with
a proof.

\begin{algorithm}[H]
\caption{(1+1)-ES with $1/5$-success rule}
\label{algo}
\begin{algorithmic}[1]
\STATE{\textbf{input} $m_0 \in \mathbb{R}^d$, $\sigma_0 > 0$, $f: \R^d \to \R$}, \textbf{parameter} $\alpha > 1$
\FOR {$t = 1,2,\dots$, \textit{until stopping criterion is met}}
	\STATE {sample $x_t \sim \Normal(m_t, \sigma_t^2 I)$}
	\IF {$f\big(x_t\big) \leq f\big(m_t\big)$}  
		\STATEWITHCOMMENT{$m_{t+1} \leftarrow x_t$}{move to the better solution}
		\STATEWITHCOMMENT{$\sigma_{t+1} \leftarrow \sigma_t \cdot \alpha$}{increase the step size}
	\ELSE 
		\STATEWITHCOMMENT{$m_{t+1} \leftarrow m_t$}{stay where we are}
		\STATEWITHCOMMENT{$\sigma_{t+1} \leftarrow \sigma_t \cdot \alpha^{-1/4}$}{decrease the step size}
	\ENDIF
\ENDFOR
\end{algorithmic}
\end{algorithm}

We consider the (1+1)-ES as specified in Algorithm~\ref{algo}. This
version of the method can be attributed to Kern et al.\ \cite{kern:2004}.
It was used in the recent analysis \cite{akimoto2018drift}.
For a given algorithm state $(m, \sigma)$, we define the success
probability $p_\text{succ}(m, \sigma) = \Pr\big(f(x) \leq f(m)\big)$.
It plays a key role for analyzing step size adaptation in the (1+1)-ES.

\section{Saddle Points}

In the following, we define various types of critical points of a
continuously differentiable objective function $f : \R^d \to \R$.
A point $x^* \in \R^d$ is called \emph{critical} if $\nabla f(x^*) = 0$,
and \emph{regular} otherwise. A critical point is a \emph{local
minimum/maximum} if there exists $r > 0$ such that it is minimal/maximal
within an open ball $B(x^*, r)$. If $x^*$ is critical but neither
(locally) minimal nor maximal, then it is a \emph{saddle point}.

If $f$ is twice continuously differentiable then most critical points
are well characterized by their second order Taylor expansion
$$ f(x) = f(x^*) + (x - x^*)^T H (x - x^*) + o(\|x - x^*\|^2) \, . $$
The eigenvalues of the Hessian $H$ determine its type: if all
eigenvalues are positive/negative then it is a minimum/maximum. If both
positive and negative eigenvalues exist then it is a saddle point. Zero
eigenvalues are not informative, since the behavior of the function in
the corresponding eigenspaces is governed by higher order terms.%
\footnote{
  It should be noted that a few interesting cases exist for zero
  eigenvalues (which should be improbable in practice), like the
  ``Monkey saddle'' $f(x) = x_1^3 - 3 x_1 x_2^2$. We believe that this
  case can be analyzed with the same techniques as developed below, but
  it is outside the scope of this paper.
}

Therefore, a prototypical problem exhibiting a saddle point is the
family of objective functions
$$ f_a(x) = \sum_{i=1}^d a_i x_i^2 $$
with parameter $a \in \R^d$. We assume that there exists
$b \in \{1, \dots, d-1\}$ such that $a_i < 0$ for all $i \leq b$ and
$a_i > 0$ for all $i > b$. In all cases, the origin $x^* = 0$ is a
saddle point. The eigenvalues of the Hessian are the parameters $a_i$.
Therefore, every saddle point of a twice continuously differentiable
function with non-zero eigen values of the Hessian is well approximated
by an instance of $f_a$ after applying translation and rotation
operations, to which the (1+1)-ES is invariant. This is why analyzing
the (1+1)-ES on $f_a$ covers an extremely general case.%

We observe that $f_a$ is scale invariant, see also
Figure~\ref{figure:f_a}: $f_a(c \cdot x) = c^2 \cdot f_a(x)$ holds, and
hence $f_a(x) < f_a(x') \Leftrightarrow f_a(c \cdot x) < f(c \cdot x')$
for all $x, x' \in \R^d$ and $c > 0$. This means that level sets look
the same on all scales, i.e., they are scaled versions of each other.
Also, the $f$-ranking of two points $x, x' \in \R^d$ agrees with the
ranking of the $c \cdot x$ versus $c \cdot x'$.

Related to the structure of $f_a$ we define the following notation. For
$x \in \R^d$ we define $x_-, x_+ \in \R^d$ as the projections of $x$
onto the first $b$ components and onto the last $d-b$ components,
respectively. To be precise, we have $(x_-)_i = x_i$ for
$i \in \{1, \dots, b\}$ and $(x_+)_i = x_i$ for
$i \in \{b+1, \dots, d\}$, while the remaining components of both
vectors are zero. We obtain $x = x_- + x_+$.

\begin{figure}[t]
\begin{center}
	\includegraphics[width=0.3\textwidth]{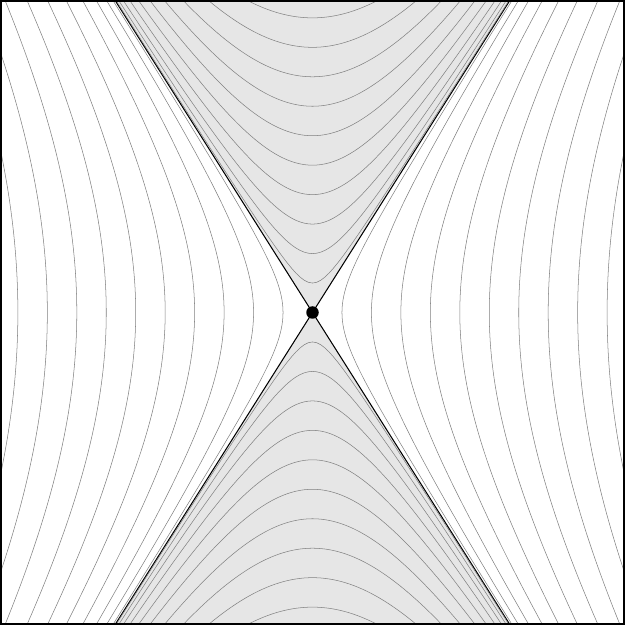}~\includegraphics[width=0.3\textwidth]{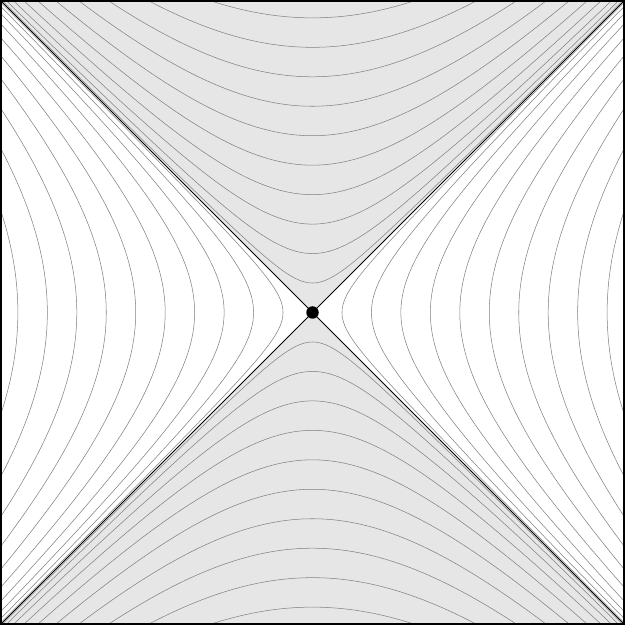}~\includegraphics[width=0.3\textwidth]{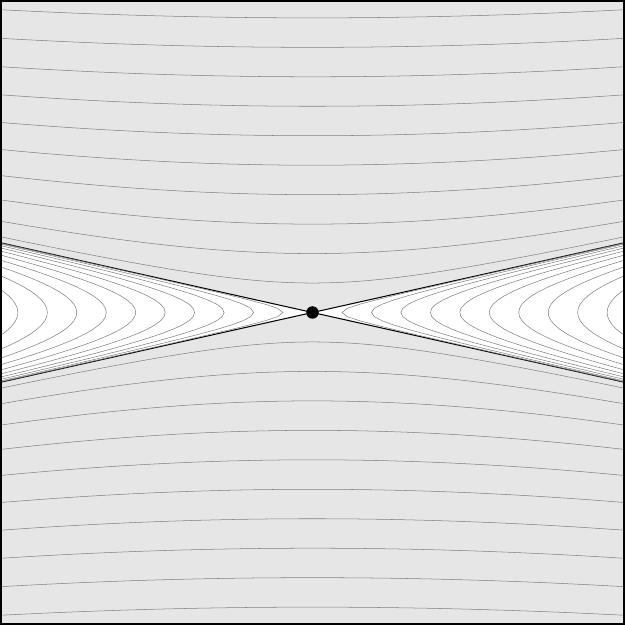}
	\caption{
		Level sets of different instances of $f_a$ for $a = (-4, 1)$
		(left), $a = (-1, 1)$ (middle), and $a = (-1, 20)$ (right),
		centered onto the saddle point. The scale of the axes is
		irrelevant since the problem is scale-invariant. The shaded
		areas correspond to positive function values. Problem
		difficulty increases from left to right, since the probability
		of sampling a ``white'' point (negative function value) in the
		vicinity of the saddle point shrinks.
		\label{figure:f_a}
	}
\end{center}
\end{figure}

For the two-dimensional case, three instances are plotted in
Figure~\ref{figure:f_a}.
The parameter $a$ controls the difficulty of the problem. The success
probability of the (1+1)-ES at the saddle point $m = 0$ equals
$p_\text{succ}(0, \sigma) = \cot^{-1}(\sqrt{|a_2/a_1|})$, which decays
to zero for $|a_2| \gg |a_1|$. This is a potentially fatal problem for
the (1+1)-ES, since it may keep shrinking its step size and converge
prematurely~\cite{glasmachers2020global}.

The contribution of this paper is to prove that we do not need to worry
about this problem. More technically precise, we aim to establish the
following theorem:

\begin{theorem}
\label{theorem:saddle}
Consider the sequence of states $(m_t, \sigma_t)_{t \in \N}$ of the
(1+1)-ES on the function $f_a$. Then, with full probability, there
exists $T \in \N$ such that for all $t \geq T$ it holds $f_a(m_t) < 0$.
\end{theorem}

It ensures that the (1+1)-ES surpasses the saddle point with full
probability in finite time (iteration $T$). This implies in particular
that the saddle point is not a limit point of the sequence
$(m_t)_{t \in \N}$ (see also Lemma~\ref{lemma:negative} below).

\section{Preliminaries}

In this section, we prepare definitions and establish auxiliary results.
We start by defining the following sets:
$D_a^- = f_a^{-1}(\R_{<0})$, $D_a^0 = f_a^{-1}(\{0\})$, and
$D_a^+ = f_a^{-1}(\R_{>0})$. They form a partition of the search
space~$\R^d$.

For a vector $x \in \R^d$ we define the semi-norms
$$ \|x\|_- = \sqrt{-\sum_{i=1}^b a_i x_i^2} \qquad \text{and} \qquad \|x\|_+ = \sqrt{\sum_{i=b+1}^d a_i x_i^2} \, . $$
The two semi-norms are Mahalanobis norms in the subspaces spanned by
eigenvectors with negative and positive eigenvalues of the Hessian of
$f_a$, respectively, when interpreting the Hessian with negative
eigenvalues flipped to positive as an inverse covariance matrix.
In other words, $f_a(x) = \|x\|_+^2 - \|x\|_-^2$ holds. Furthermore, we
have $\|x_+\|_+ = \|x\|_+$, $\|x_-\|_- = \|x\|_-$, $\|x_-\|_+ = 0$, and
$\|x_+\|_- = 0$.

In the following, we exploit scale invariance of $f_a$ by analyzing the
stochastic process $(m_t, \sigma_t)$ in a normalized state space. We map
a state to the corresponding normalized state by
$$ (m, \sigma) \mapsto \left( \frac{m}{\|m\|_+}, \frac{\sigma}{\|m\|_+} \right) = (\tilde m, \tilde \sigma) \, . $$
This normalization is different from the normalizations $m/\sigma$ and
$m/(d \sigma)$, which give rise to a scale-invariant process when
minimizing the Sphere function \cite{akimoto2018drift}. The
different normalization reflects the quite different dynamics of the
(1+1)-ES on~$f_a$.

We are particularly interested in the case $m \in D_a^+$, since we need
to exclude the case that the (1+1)-ES stays in that set indefinitely.
Due to scale invariance, this condition is equivalent to
$\tilde m \in D_a^+$. We define the set
$$ M = \big\{x \in \R^d \,\big|\, \|x\|_+ = 1 \big\} \, . $$
The state space for the normalized states $(\tilde m, \tilde \sigma)$
takes the form $M \times \R_{>0}$. We also define the subset
$M^+_0 = M \cap (D_a^+ \cup D_a^0)$. The reason to include the zero
level set is that closing the set makes it compact. Its boundedness can
be seen from the reformulation
$M^+_0 = \big\{m \in \R^d \,\big|\, \|m\|_+ = 1 \text{ and } \|m\|_- \leq 1 \big\}$.
In the following, compactness will turn out to be very useful,
exploiting the fact that on a compact set, every lower semi-continuous
function attains its infimum.

The success probability $p_\text{succ}(m, \sigma)$ is scale invariant,
and hence it is well-defined as a function of the normalized state
$(\tilde m, \tilde \sigma)$. It is everywhere positive. Indeed, it is
uniformly lower bounded by $p_{\min} = \min(p^*, \frac12) > 0$, where
$p^* = p_\text{succ}(0, 1)$ denotes the success probability in the
saddle point (which is independent of the step size, and depends only
on~$a$).
The following two lemmas deal with the success rate in more detail.

\begin{lemma}
\label{lemma:negative}
If there exists $T \in \N$ such that $m_T \in D_a^0 \cup D_a^-$ then
with full probability, the saddle point $0 \in \R^d$ of $f_a$ is not a
limit point of the sequence $(m_t)_{t \in \N}$.
\end{lemma}
\begin{proof}
Due to elitism, the sequence $m_t$ can jump from $D_a^+$ to $D_a^0$ and
then to $D_a^-$, but not the other way round. In case of $m_T \in D_a^-$
all function values for $t > T$ are uniformly bounded away from zero by
$f(m_t) \leq f(m_T) < 0$. Therefore $f(m_t)$ cannot converge to zero,
and $m_t$ cannot converge to the saddle point.

Now consider the case $m_T \in D_a^0$. For all $m \in D_a^0$ and all
$\sigma > 0$, the probability of sampling an offspring in $D_a^-$ is
positive, and it is lower bounded by $p_{\min}$, which is positive and
independent of $\sigma$. Not sampling an offspring $m_t \in D_a^-$ for
$n$ iterations in a row has a probability of at most $(1 - p_{\min})^n$,
which decays to zero exponentially quickly. Therefore, with full
probability, we obtain $m_t \in D_a^-$ eventually.
\qed
\end{proof}

However, $p_{\min}$ being positive is not necessarily enough for the
(1+1)-ES to escape the saddle point, since for $p_{\min} < 1/5$ it may
stay inside of $D_a^+$, keep shrinking its step size, and converge
prematurely \cite{glasmachers2020global}. In fact, based on the choice
of the parameter $a$ of $f_a$, $p_{\min}$ can be arbitrarily small. In
the following lemma, we therefore prepare a drift argument, ensuring
that the normalized step size remains in or at least always returns to a
not too small value.

\begin{lemma}
\label{lemma:40-percent-success-rate}
There exists a constant $0 < \tilde \sigma_{\text{40\%}} \leq \infty$
such that
$p_\text{succ}(\tilde m, \tilde \sigma) \geq 2/5$ holds for all states
fulfilling $\tilde m \in M^+_0$ and
$\tilde \sigma \leq \tilde \sigma_{\text{40\%}}$.
\end{lemma}

\begin{proof}
It follows immediately from the geometry of the level sets (see also
Figure~\ref{figure:f_a}) that for each fixed $\tilde m \in M^+_0$
(actually for $m \not= 0$), it holds
$$ \lim\limits_{\tilde \sigma \to 0} p_\text{succ}(\tilde m, \tilde \sigma) = \frac12
   \qquad \text{and} \qquad
   \lim\limits_{\tilde \sigma \to \infty} p_\text{succ}(\tilde m, \tilde \sigma) = p^* \, . $$
Noting that $p_\text{succ}(\tilde m, \tilde \sigma)$ is continuous
between these extremes, we define a pointwise critical step size as
$$ \tilde \sigma_{\text{40\%}}(\tilde m) = \arg\min_{\tilde \sigma > 0} \big\{ p_\text{succ}(\tilde m, \tilde \sigma) \leq 2/5 \big\} \, . $$
With the convention that $\arg\min$ over an empty set is $\infty$, this
definition makes
$\tilde \sigma_{\text{40\%}} : M^+_0 \to \R \cup \{\infty\}$ a lower
semi-continuous function. Due to compactness of $M^+_0$ it attains its
minimum $\tilde \sigma_{\text{40\%}} > 0$.
\qed
\end{proof}

\section{Drift of the Normalized State}

In this section we establish two drift arguments. They apply to the
following drift potential functions:
\begin{align*}
V   (\tilde m, \tilde \sigma) &= \log(\tilde \sigma)
\\
W   (\tilde m, \tilde \sigma) &= \|\tilde m\|_-
\\
\Phi(\tilde m, \tilde \sigma) &= \beta \cdot V(\tilde m, \tilde \sigma) + W(\tilde m, \tilde \sigma)
\end{align*}
The potentials govern the dynamics of the step size $\tilde \sigma$, of
the mean $\tilde m$, and of the combined process, namely the (1+1)-ES.
The trade-off parameter $\beta > 0$ will be determined later. Where
necessary we extend the definitions to the original state by plugging in
the normalization, e.g., resulting in
$W(m, \sigma) = \frac{\|m\|_-}{\|m\|_+}$.

For a normalized state $(\tilde m, \tilde \sigma)$ let
$(\tilde m', \tilde \sigma')$ denote the normalized successor state.
We measure the drift of all three potentials as follows:
\begin{align*}
\Delta^V   (\tilde m, \tilde \sigma) &= \Expectation \big[ V(\tilde \sigma') - V(\tilde \sigma) \big]
\\
\Delta^W   (\tilde m, \tilde \sigma) &= \Expectation \big[ \min\{W(\tilde m') - W(\tilde m), 1\} \big]
\\
\Delta^\Phi(\tilde m, \tilde \sigma) &= \beta \cdot \Delta^V(\tilde m, \tilde \sigma) + \Delta^W(\tilde m, \tilde \sigma)
\end{align*}

As soon as $W(\tilde m) > 1$, $\tilde m \in D_a^-$ holds and the
(1+1)-ES has successfully passed the saddle point according to
Lemma~\ref{lemma:negative}. Therefore we aim to show that the sequence
$W(\tilde m_t)$ keeps growing, and that is passes the threshold of one.
To this end, we will lower bound the progress $\Delta^W$ of the
truncated process.

Truncation of particularly large progress in the definition of
$\Delta^W$, i.e., $W$-progress larger than one, serves the purely
technical purpose of making drift theorems applicable. This sounds
somewhat ironic, since a progress of more than one on $W$ immediately
jumps into the set $D_a^-$ and hence passes the saddle. On the technical
side, an upper bound on single steps is a convenient prerequisite. Its
role is to avoid that the expected progress is achieved by very few
large steps while most steps make no or very litte progress, which would
make it impossible to bound the runtime based on expected progress. Less
strict conditions allowing for rare large steps are possible
\cite{hajek1982hitting,lehre2013general}. The technique of bounding the
single-step progress instead of the domain of the stochastic process was
introduced in \cite{akimoto2018drift}.

The speed of the growth of $W$ turns out to depend on $\tilde \sigma$.
In order to guarantee growth at a sufficient pace, we need to keep the
normalized step size from decaying to zero too quickly. Indeed, we will
show that the normalized step size drifts away from zero by analyzing
the step-size progress~$\Delta^V$.

The following two lemmas establish the drift of mean $\tilde m$ and step
size~$\tilde \sigma$.

\begin{lemma}
\label{lemma:step-size-drift}
Assume $\tilde m \in M^+_0$.
There exists a constant $B_1$ such that
$\Delta^V(\tilde m, \tilde \sigma) \geq B_1$ holds.
Furthermore, there exist constants $B_2 > 0$ and
$\tilde \sigma^* \in (0, \tilde \sigma_{\text{40\%}}]$ such that for all
$\tilde \sigma \leq \tilde \sigma^*$ it holds
$\Delta^V(\tilde m, \tilde \sigma) \geq B_2$.
\end{lemma}

\begin{lemma}
\label{lemma:mean-drift}
Assume $\tilde m \in M^+_0$.
The $W$-progress $\Delta^W(\tilde m, \tilde \sigma)$ is everywhere
positive. Furthermore, for each
$\tilde \sigma^* \in (0, \tilde \sigma_{\text{40\%}}]$ there exists a
constant $C > 0$ depending on $\tilde \sigma^*$ such that it holds
$\Delta^W(\tilde m, \tilde \sigma) \geq C$ if
$\tilde \sigma \geq \tilde \sigma^*$.
\end{lemma}

The proofs of these lemmas contain the main technical work.


\begin{proof}[of Lemma~\ref{lemma:step-size-drift}]
From the definition of $\tilde \sigma_{\text{40\%}}$,
for $\tilde \sigma \leq \tilde \sigma_{\text{40\%}}$,
we conclude that the probability of sampling a successful offspring is
at least $2/5$.
In case of an unsuccessful offspring, $\tilde \sigma$ shrinks by the
factor $\alpha^{-1/4}$. For a successful offspring it is multiplied by
$\alpha \cdot \frac{\|m\|_+}{\|m'\|_+}$, where the factor $\alpha > 1$
comes from step size adaptation, and the fraction is due to the
definition of the normalized state.

The dependency on $m$ and $m'$ is inconvenient. However, for small step
size $\tilde \sigma$ we have $\|m'\| \approx \|m\|$, simply because
modifying $m$ with a small step results in a similar offspring, which is
then accepted as the new mean $m'$. In the limit we have
$$ \lim_{\tilde \sigma \to 0} \Expectation \left[ \log \left( \frac{\|m\|_+}{\|m'\|_+} \right) \right] = 0 \, . $$
This allows us to apply the same technique as in the proof of
Lemma~\ref{lemma:40-percent-success-rate}. The function
$(\tilde m, \tilde \sigma) \mapsto \Expectation \left[ \log \left( \frac{\|m\|_+}{\|m'\|_+} \right) \right]$
is continuous. We define a pointwise lower bound through the lower
semi-continuous function
$$ \tilde m \mapsto \arg\min_{0 < \tilde \sigma \leq \tilde \sigma_{\text{40\%}}} \left\{ \Expectation \left[ \log \left( \frac{\|m\|_+}{\|m'\|_+} \right) \right] \leq \frac{1}{\sqrt{\alpha}} \right\} \, , $$
where the $\arg\min$ over the empty set shall take the value
$\sigma_{\text{40\%}}$. We define $\tilde \sigma^*$ as its infimum. It
is attained, since $M^+_0$ is compact, and hence positive.

For $\tilde \sigma \leq \tilde \sigma^*$ we obtain the
following drift condition:
\begin{align*}
	\Delta^V(\tilde m, \tilde \sigma)
		& \geq \frac25 \cdot \left[ \log(\alpha^{-\frac12}) + \log(\alpha) \right] - \left(1 - \frac25\right) \cdot \frac14 \cdot \log(\alpha)
		\\
		& = \frac15 \cdot \log(\alpha) - \frac{3}{20} \cdot \log(\alpha)
		= \frac{1}{20} \cdot \log(\alpha) > 0
\end{align*}
For $\tilde \sigma > \tilde \sigma^*$ we consider the worst case of a
success rate of zero. Then we obtain
$$ \Delta^V(\tilde m, \tilde \sigma) \geq -\frac14 \cdot \log(\alpha) \, . $$
Hence, the statement holds with $B_1 = -\frac14 \cdot \log(\alpha)$ and
$B_2 = \frac{1}{20} \cdot \log(\alpha)$.
\qed
\end{proof}


\begin{proof}[of Lemma~\ref{lemma:mean-drift}]
We start by showing that $\Delta^W$ is always positive. We decompose the
domain of the sampling distribution (which is all of $\R^d$) into
spheres of fixed radius $r = \|\tilde m' - \tilde m\|$ and show that the
property holds, conditioned to the success region within each sphere.
Within each sphere, the distribution is uniform.

\begin{figure}
\begin{center}
	\includegraphics[width=\textwidth]{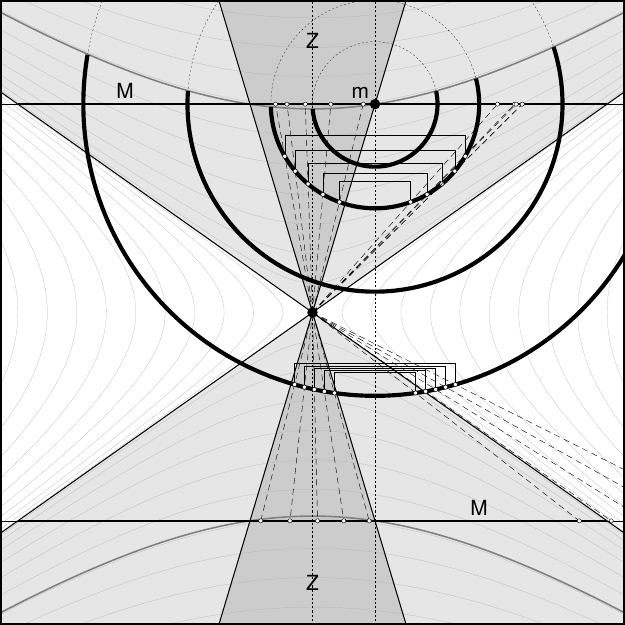}
	\caption{
		Geometric illustration of the proof of
		Lemma~\ref{lemma:mean-drift}. The figure shows the saddle
		point (center), level sets of $f_a$ (thin lines), the point $m$,
		and the set $M$ (two horizontal lines), the thick part of which
		is $M^+_0$. The area $D_a^-$ has a white background, while
		$D_a^+$ is the gray area. The dark gray area is the set $Z$.
		\\[0.25em]
		The figure displays spheres of different radii into which the
		sampling distribution is decomposed. The spheres are drawn as
		dotted circles, and as bold solid arcs in the region of
		successful offspring, outperforming $m$. The thickened arcs
		indicate sets of corresponding points. Ten pairs of
		corresponding points are shown, five each for two different
		spheres.
		\label{figure:proof-mean}
	}
\end{center}
\end{figure}

Each sphere makes positive and negative contributions to
$W(x) - W(\tilde m) = \frac{\|x\|_-}{\|x\|_+} - \|\tilde m\|_-$. Within
the set
$$ Z = \left\{z \in \R^d \,\left|\, \frac{\|z\|_-}{\|z\|_+} < \|\tilde m\|_- \right.\right\} $$
the contributions are negative. The set is illustrated in
Figure~\ref{figure:proof-mean}. Outside of $Z$, contributions are
positive. We aim to show that overall, for each sphere, the expectation
is positive. To this end, we define pairs of corresponding points such
that the negative contribution of one point is (more than) compensated
by the positive contribution of the other.
For our argument, it is important that the Lebesgue measure of each
subset $S \subset Z$ is at most as large the Lebesgue measure of the
set of corresponding points outside of $Z$. This property will be
fulfilled by construction, and with equality.

For each successful offspring in $Z$ we define a corresponding point
outside of $Z$ on the same sphere. Corresponding points are mirrored at
the symmetry axis through $m$. More precisely, for $z \in Z$ we define
$z'_- = 2 m_- - z_-$ and $z'_+ = z_+$. It holds
$\tilde m_- - z_- = z'_- - \tilde m_-$, and we call this difference
$\delta = \tilde m_- - z_-$.

By projecting both points onto the normalized state space $M$ we obtain
their contributions to the expectation. This amounts to following the
dashed lines in Figure~\ref{figure:proof-mean}.
Adding the contributions of $z$ and $z'$ yields
\begin{align*}
	W(z) + W(z') - 2 W(\tilde m)
		&= \frac{\|z\|_-}{\|z\|_+} + \frac{\|z'\|_-}{\|z'\|_+} - 2 \|\tilde m\|_- \\
		&= \frac{\|\tilde m - \delta\|_- + \|\tilde m + \delta\|_-}{\|z\|_+} - 2 \|\tilde m\|_- \\
		&\geq \|\tilde m - \delta\|_- + \|\tilde m + \delta\|_- - 2 \|\tilde m\|_-
		\geq 0 \, .
\end{align*}
The first inequality holds because of $\|z\|_+ = \|z'\|_+ \leq 1$ (note
that the set $M$ corresponds to $\|\cdot\|_+ = 1$, see also
Figure~\ref{figure:proof-mean}). The second step is the triangle
inequality of the semi-norm $\|\cdot\|_-$. Both inequalities are strict
outside of a set of measure zero.

Truncating progress $W(z') - W(\tilde m)$ larger than one does not pose
a problem. This is because $W(\tilde m) - W(z) < 1$ is obtained from the
fact that $\tilde m$ and $z$ are both contained in $D_a^+$, and this is
where $W$ takes values in the range $[0, 1)$. We obtain
$W(z) - W(\tilde m) + 1 > 0$ in the truncated case.

Integrating the sum over all corresponding pairs on the sphere, and
noting that there are successful points outside of $Z$ which do not have
a successful corresponding point inside but not the other way round, we
see that the expectation of $W(\tilde m') - W(\tilde m)$ over the
success region of each sphere is positive.

Integration over all radii $r > 0$ completes the construction. In the
integration, the weights of different values of $r$ depend on
$\tilde \sigma$ (by means of the pdf of a $\chi$-distribution scaled by
$\tilde \sigma$). Since the integrand is non-negative, we conclude that
$\Delta^W(\tilde m, \tilde \sigma) > 0$ holds for all
$\tilde \sigma > 0$.

In the limit $\tilde \sigma \to \infty$, the expected progress in case
of success converges to one (due to truncation), and hence the expected
progress converges to $p^*$. This allows us to exploit compactness once
more. The expectation of the truncated progress
$\Delta^W(\tilde m, \tilde \sigma)$ is continuous as a function of the
normalized state. We define a pointwise lower bound as
$$ C(\tilde m) = \min_{\tilde \sigma \geq \tilde \sigma_{\text{40\%}}} \Big\{ \Delta^W(\tilde m, \tilde \sigma) \Big\} \, . $$
$C(\tilde m)$ is a continuous function, and (under slight misuse of
notation) we define $C$ as its infimum over the compact set $M^+_0$.
Since the infimum is attained, it is positive.
\qed
\end{proof}


Now we are in the position to prove the theorem.
\begin{proof}[of Theorem~\ref{theorem:saddle}]
Combining the statements of Lemma \ref{lemma:step-size-drift} and \ref{lemma:mean-drift}
we obtain
$$\Delta^\Phi(\tilde m, \tilde \sigma) \geq \theta := \min\{\beta B_2, C + \beta B_1\} \, $$
for all $\tilde m \in M^+_0$ and $\tilde \sigma > 0$. The choice
$\beta = \frac{-C}{2 B_1}$ results in $\theta = \min\{B_2, C/2\} > 0$.
The constant $\theta$ is a bound on the additive drift of $\Phi$, hence
we can apply additive drift with tail bound (e.g., Theorem~2 in
\cite{lehre2013general} with additive drift as a special case, or
alternatively inequality (2.9) in Theorem 2.3 in
\cite{hajek1982hitting}) to obtain the following: Let
$$ T = \min\Big\{t \in \N \,\Big|\, \Phi(\tilde m_t, \tilde \sigma_t) > 1\Big\} $$
denote the waiting time for the event that $\Phi$ reaches or exceeds
one (called the first hitting time). Then the probability of $T$
exceeding $T_0 \in \N$ decays exponentially in $T_0$. Therefore, with
full probability, the hitting time $T$ is finite.
$\Phi(\tilde m_T, \tilde \sigma_T) > 1$ is equivalent to $f(m_T) < 0$.
For all $t > T$, the function value stays negative, due to elitism.
\qed
\end{proof}

\section{Discussion and Conclusion}

We have established that the (1+1)-ES does not get stuck at a
(quadratic) saddle point, irrespective of its conditioning (spectrum of
its Hessian), with full probability. This is all but a trivial result
since the algorithm is suspectable to premature convergence if the
success rate is smaller than $1/5$. For badly conditioned problems,
close to the saddle point, the success rate can indeed be arbitrarily
low. Yet, the algorithm passes the saddle point by avoiding it
``sideways'': While approaching the level set containing the saddle
point, there is a systematic sidewards drift away from the saddle. This
keeps the step size from decaying to zero, and the saddle is
circumvented.

In this work we are only concerned with quadratic functions. Realistic
objective functions to be tackled by evolution strategies are hardly
ever so simple. Yet, we believe that our analysis is of quite general
value. The reason is that the negative case, namely premature
convergence to a saddle point, is an inherently local process, which is
dominated by a local approximation like the second order Taylor
polynomial around the saddle point. Our analysis makes clear that as
long as the saddle is well described by a second order Taylor
approximation with a full-rank Hessian matrix, then the (1+1)-ES will
not converge prematurely to the saddle point. We believe that our result
covers the most common types of saddle points. Notable exceptions are
sharp ridges, plateaus, and Monkey saddles.

The main limitation of this work is not the covered class of functions,
but the covered algorithms. The analysis sticks closely to the (1+1)-ES
with its success-bases step size adaptation mechanism. There is no
reason to believe that a fully fledged algorithm like the covariance
matrix adaptation evolution strategy (CMA-ES) \cite{hansen:2001} would
face more problems with a saddle than the simple (1+1)-ES, and to the
best of our knowledge, there is no empirical indication thereof. In
fact, our intuition is that most algorithms should profit from the
sidewards drift, as long as they manage to break the symmetry of the
problem, e.g., through randomized sampling. Yet, it should be noted that
our analysis does not easily extend to non-elitist algorithms and step
size adaptation methods other than success-based rules.

\end{document}